\documentclass[11pt]{article}

\usepackage{footnote}
\usepackage{natbib}
\usepackage{csquotes}

\usepackage{bbm}
\usepackage[top=3cm,bottom=2cm,left=3cm,right=3cm]{geometry}

\usepackage{braket,amsfonts}

\usepackage{array}

\usepackage[caption=false]{subfig}

\usepackage{pgfplots}

\usepackage{algorithmic}

\usepackage{graphicx,epstopdf}

\usepackage{amsopn}

\usepackage{xspace}
\usepackage{bold-extra}
\usepackage[most]{tcolorbox}

\colorlet{texcscolor}{blue!50!black}
\colorlet{texemcolor}{red!70!black}
\colorlet{texpreamble}{red!70!black}
\colorlet{codebackground}{black!25!white!25}

\title{Generative Models with ELBOs Converging to Entropy Sums}

\author{
\ \\[7mm]
Jan Warnken$^{1,}$\thanks{joint major contributions}\hspace{3.3mm} Dmytro Velychko$^1$\hspace{2mm} Simon Damm$^2$\\[1.0ex]
Asja Fischer$^2$\hspace{2mm} J\"org L\"ucke$^{1,}$\footnotemark[1]\\[9.5mm]
$^1$\small{Machine Learning, Faculty VI, University of Oldenburg, Germany}\\
$^2$\small{Machine Learning, Department of Computer Science, Ruhr University Bochum, Germany}
}

\date{}
\definecolor{white}{rgb}{1.0,1.0,1.0}
\definecolor{brightred}{rgb}{1.0,0.1,0.1}
\definecolor{brightblue}{rgb}{0.0,0.0,0.8}
\definecolor{darkblue}{rgb}{0.0,0.0,0.5}
\definecolor{darkgreen}{rgb}{0.0,0.3,0.0}
\definecolor{brightgreen}{rgb}{0.0,0.8,0.0}
\definecolor{darkblack}{rgb}{0.0,0.0,0.0}
\definecolor{grey}{rgb}{0.3,0.3,0.3}

\graphicspath{{./}}

\usepackage{amsthm}
\newtheorem{proposition}{Proposition}
\theoremstyle{definition}
\newtheorem{definition}{Definition}

\definecolor{white}{rgb}{1.0,1.0,1.0}
\definecolor{brightred}{rgb}{1.0,0.1,0.1}
\definecolor{brightblue}{rgb}{0.0,0.0,0.8}
\definecolor{darkblue}{rgb}{0.0,0.0,0.5}
\definecolor{darkgreen}{rgb}{0.0,0.3,0.0}
\definecolor{brightgreen}{rgb}{0.0,0.8,0.0}
\definecolor{darkblack}{rgb}{0.0,0.0,0.0}
\definecolor{grey}{rgb}{0.3,0.3,0.3}

\long\def\comment#1{}  

\newcommand{\disT}{\textstyle}
\newcommand{\disS}{\displaystyle}

\DeclareMathOperator{\diag}{diag}

\newcommand{\undersets}[2]{\underset{\hbox to 0pt{\hss{\scriptsize #1}\hss}}{#2}}
\newcommand{\oversets}[2]{\overset{\hbox to 0pt{\hss{\scriptsize #1}\hss}}{#2}}

\newcommand{\FF}{{\mathcal F}}
\newcommand{\FFt}{\tilde{\FF}}

\newcommand{\HH}{{\mathcal H}}
\newcommand{\HHt}{\tilde{\HH}}

\newcommand{\LL}{{\mathcal L}}

\newcommand{\NN}{{\mathcal N}}

\newcommand{\RRR}{\mathbbm{R}}
\newcommand{\Scal}{{\mathcal S}}  

\newcommand{\Jnew}{\frac{\partial\etaVec(\zVec;\ThetaVec)}{\partial\thetaVec^{\mathrm{\,T}}}}

\newcommand{\Inew}{\frac{\partial\zetaVec(\PsiVec)}{\partial\PsiVec^{\mathrm{\,T}}}}

\newcommand{\sit}{\tilde{\sigma}}

\newcommand{\EEE}[2]{\mathbbm{E}_{#1}\big\{#2\big\}}

\def\Vhrulefill{\leavevmode\leaders\hrule height 0.7ex depth \dimexpr0.4pt-0.7ex\hfill\kern0pt}

\newcommand{\One}{\hbox{$1\hskip -1.2pt\vrule depth 0pt height 1.6ex width 0.7pt\vrule depth 0pt height 0.3pt width 0.12em$}}

\newcommand{\OneDD}{\One_{_{D\times{}D}}}

\newcommand{\ZeroLL}{0_{_{L\times{}L}}}

\newcommand{\muVec}{\vec{\mu}}

\newcommand{\bVec}{\vec{b}}

\newcommand{\qn}{q^{(n)}}

\newcommand{\qBar}{\overline{q}}
\newcommand{\qbar}{\overline{q}}

\newcommand{\classk}{k}

\newcommand{\subc}{c}

\newcommand{\xVec}{\vec{x}}
\newcommand{\zVec}{\vec{z}}
\newcommand{\xVecN}{\vec{x}^{\hspace{0.5pt}(n)}}
\newcommand{\aVec}{\vec{a}}

\newcommand{\wVec}{\vec{w}}
\newcommand{\piVec}{\vec{\pi}}

\newcommand{\lambdaVec}{\vec{\lambda}}

\newcommand{\PsiVec}{\vec{\Psi}}
\newcommand{\PsiVecT}{\PsiVec^{\mathrm{\,T}}}

\newcommand{\pT}{p_{\footnotesize{\ThetaVec}}}

\newcommand{\pPsi}{p_{\footnotesize{\PsiVec}}}

\newcommand{\qPhiN}{q_{\Phi}^{(n)}}

\newcommand{\zetaVec}{\vec{\zeta}}
\newcommand{\zetaVecT}{\vec{\zeta}^{\mathrm{\,T}}}

\NewDocumentCommand{\y}{O{}O{}}{y_{#2}^{\,#1}{}}
\NewDocumentCommand{\yVec}{O{}}{\vec{y}_{\vphantom{c}}^{\,#1}{}}

\NewDocumentCommand{\Wgen}{O{}O{}}{\mathcal{W}_{#1#2}}

\NewDocumentCommand{\Rgen}{O{}O{}}{\mathcal{R}_{#1#2}}

\NewDocumentCommand{\W}{O{}O{}}{W_{#1#2}}

\NewDocumentCommand{\R}{O{}O{}}{R_{#1#2}}

\NewDocumentCommand{\Sc}{O{\subc}O{}}{s_{#1}^{#2}}

\NewDocumentCommand{\sVec}{O{}}{\vec{s}_{\vphantom{\subc}}^{\,#1}}

\NewDocumentCommand{\Igenc}{O{\subc}O{}}{\mathcal{I}_{#1}^{#2}}

\NewDocumentCommand{\Ic}{O{\subc}O{}}{I_{#1}^{#2}}

\NewDocumentCommand{\tk}{O{\classk}O{}}{t_{#1}^{#2}}

\NewDocumentCommand{\tVec}{O{}}{\vec{t}_{\vphantom{\classk}}^{\,#1}}

\NewDocumentCommand{\eps}{O{}}{\epsilon_{\textnormal{\tiny $#1$}}}

\NewDocumentCommand{\epst}{O{}}{\tilde{\epsilon}_{\textnormal{\tiny $#1$}}}

\newcommand{\Bern}{\textrm{Bern}}

\newcommand{\etaVec}{\vec{\eta}}
\newcommand{\etaVecT}{\vec{\eta}^{\mathrm{T}}}
\newcommand{\TVec}{\vec{T}}

\newcommand{\ThetaVec}{\vec{\Theta}}
\newcommand{\ThetaVecT}{{\ThetaVec}^{\mathrm{\,T}}}

\newcommand{\thetaVec}{\vec{\theta}}
\newcommand{\thetaVecT}{\vec{\theta}^{\mathrm{\,T}}}

\newcommand{\alphaVec}{\vec{\alpha}}
\newcommand{\betaVec}{\vec{\beta}}

\newcommand{\delll}[2]{\frac{\partial{}#1}{\partial{}#2}}
\newcommand{\del}[1]{\frac{\partial}{\partial{}#1}}

\begin{document}

\maketitle

\begin{abstract}
\noindent The evidence lower bound (ELBO) is one of the most central objectives for probabilistic unsupervised learning. For the ELBOs of several generative models and model classes, we here prove convergence to entropy sums. As one result, we provide a list
of generative models for which entropy convergence has been shown, so far, along with the corresponding expressions for entropy sums. 
Our considerations include very prominent generative models such as probabilistic PCA, sigmoid belief nets or
Gaussian mixture models. 
%
%
%
However, we treat more models and entire model classes such as general mixtures of exponential family distributions.
Our main contributions are the proofs for the individual models.
For each given model we show that the conditions stated in Theorem~1 or Theorem~2 of \cite{LuckeWarnken2024} are fulfilled such
that by virtue of the theorems the given model's ELBO is equal to an entropy sum at all stationary points.
The equality of the ELBO at stationary points applies under realistic conditions: for finite numbers of data points, for model/data mismatches, at any stationary point including saddle points etc, and it applies for any well behaved family of variational distributions. 
\end{abstract}
%
%
%
\section{Introduction}\label{SecIntro}
We consider probabilistic generative models that are defined based on an elementary directed graphical model with two sets of variables.
We use the same notation as in the paper of \citet{LuckeWarnken2024}, which provides the main theoretical tools that we here apply to concrete models:
we use $\xVec$ for the set of observables, and we use $\zVec$ for the set of latent variables. A generative model is then given by
\begin{eqnarray}
 \zVec \,\sim\, p_{\Theta}(\zVec)\,\ \ \mbox{\ and\ }\ \ \xVec \,\sim\, p_{\Theta}(\xVec\,|\,\zVec)\,, \label{EqnGenModelIntro}
\end{eqnarray}
where $\Theta$ denotes the set of model parameters. We refer to the distributions $p_{\Theta}(\zVec)$ and $p_{\Theta}(\xVec\,|\,\zVec)$
as {\em prior distribution} and  {\em observable distribution}, respectively. 
Only for mixture models we slightly divert from the notation above and use $c$ for the latent class variables instead of $\zVec$ (but this will be clear from context).

\noindent{}For $N$ data points $\xVecN$ the ELBO (or free energy) learning objective \citep{JordanEtAl1999,NealHinton1998} is given by:
\begin{align}
 \FF(\Phi,\Theta) 
%
%
%
=\frac{1}{N}\sum_{n} \int \qn_{\Phi}\!(\zVec)\, \log\!\big( p_{\Theta}(\xVecN\,|\,\zVec) \big)\hspace{0.5ex} \mathrm{d}\zVec
                 \,-\, \frac{1}{N}\sum_{n} D_{\mathrm{KL}}\big[ \qn_{\Phi}\!(\zVec),p_{\Theta}(\zVec) \big]\,, \label{EqnFFIntro}
%
\end{align}
where $\qn_{\Phi}(\zVec)$ are the {\em variational distributions} used for optimization.

\noindent{}For each here considered generative model (or for each considered model class), we will show that the ELBO (Eqn.\,\ref{EqnFFIntro}) is at all stationary points
of learning equal to (see \cite{LuckeWarnken2024} for details):
\begin{align}
%
\ \FF(\Phi,\Theta) &=\frac{1}{N}\sum_{n} \HH[\qPhiN(\zVec)]  \,-\, \HH[\,p_{\Theta}(\zVec)] \,-\, \EEE{\qbar_{\Phi}}{ \HH[\,p_{\Theta}(\xVec\,|\,\zVec)] }\,,\phantom{\small{}ix}
%
%
\label{EqnTheoremIntro}\\
%
%
%
%
%
%
%
%
%
&\mbox{where\ \ } \qbar_{\Phi}(\zVec) := \frac{1}{N} \sum_{n} \qn_{\Phi}(\zVec) \label{EqnAggPosterior}
\end{align}
is the average variational distribution. The distribution $\qbar_{\Phi}(\zVec)$ is also sometimes referred to as {\em aggregate posterior} (e.g.\,\cite{MakhzaniEtAl2015,TomczakWelling2018,AnejaEtAl2021}) but we note that it can (as an average of variational approximations) significantly divert from the average posterior.
Notably, expression (\ref{EqnTheoremIntro}) will, for each here treated generative model, apply under realistic conditions, i.e., the result will apply for finite numbers of data points, model/data mismatches, at any stationary point (global and local maxima as well as saddle points), and for any (well-behaved) family of variational distributions.

%
%
In the literature on deep generative models, the variational distributions are typically amortized, i.e., the variational distributions $\qn_{\Phi}(\zVec)$ are obtained via a function from the set of observables to latents: $\qn_{\Phi}(\zVec)=q_{\Phi}(\zVec; \xVecN)$. Furthermore, the avarage over data points is often causually abbreviated as an expectation value w.r.t.\,the true data distribution $p(\xVec)$. If the entropy sum of Eqn.\,\ref{EqnTheoremIntro} is denoted analogously, one therefore obtains the expression:
%
%
\begin{align}
%
\ \FF(\Phi,\Theta) = \EEE{p}{  \HH[q_{\Phi}(\zVec; \xVec)] } \,-\, \HH[\,p_{\Theta}(\zVec)] \,-\, \EEE{\qbar_{\Phi}}{ \HH[\,p_{\Theta}(\xVec\,|\,\zVec)] }\,.\phantom{\small{}ix}
%
%
\label{EqnTheoremIntroLimit}
%
%
%
%
%
%
\end{align}
We will consider the case where $N$ can take on any value, i.e., expression~(\ref{EqnTheoremIntro}) applies, or expression~(\ref{EqnTheoremIntroLimit}) applies, with the first term interpreted as abbreviation for the average over $N$ data points. But our results will also apply in the limit of $N\rightarrow\infty$ such that expression (\ref{EqnTheoremIntroLimit}) is also valid with the first term being an actual expectation value. In the limit, the aggregate posterior $\qbar_{\Phi}(\zVec)$ is the distribution obtained from (\ref{EqnAggPosterior}) when $N\rightarrow\infty$ (assuming that all limits exist). 

To show that the ELBO takes on the form (\ref{EqnTheoremIntro}), we will apply Theorem~1 or Theorem~2 of \citet{LuckeWarnken2024}. In virtue of the theorems,
equality of ELBO and entropy sums can be shown for a given model by verifying that the assumptions of one of the theorems are satisfied.
The corresponding theorem then implies that (\ref{EqnTheoremIntro}) holds. One main assumption that has to be fulfilled can usually directly be verified by considering
the definition of a generative model: prior distribution and observable distribution have to be exponential family distributions. We refer to such a model as exponential family generative model or {\em EF generative model}. The only non-trivial conditions for
Theorem~1 or Theorem~2 of \citet{LuckeWarnken2024} to be applicable is a parameterization condition. The reader only interested in using the entropy sum expression for the ELBO is not required to study the condition, however. Instead, it is sufficient to observe that entropy sum convergence can be shown for the generative model one may be interested in. Different generative models are listed in Secs.~\ref{SecSBN} to~\ref{SecMixtures}. For completeness and for the proofs for the treated models, we here reiterate the parameterization condition. 

Consider a generative model as specified in Eqn.~\ref{EqnGenModelIntro}. We take the prior distribution to be parameterized by parameters $\vec{\Psi}$, and we assume the observable distribution to be parameterized by another set of parameters $\vec{\Theta}$.
However, we also account for generative models, such as probabilistic PCA (\mbox{p-PCA;} \cite{Roweis1998,TippingBishop1999}), that do not involve a parameterized prior.
For both sets of parameters we arranged the scalar parameters into one column vector. The set of all parameters of the generative model we will
denote by $\Theta=(\vec{\Psi},\vec{\Theta})$ (i.e., the symbol $\Theta$ is slightly overloaded but disambiguated via vector
or non-vector notation). Prior and noise distribution we assume to be exponential family distributions. 
Then there exist mappings $\zetaVec(\PsiVec)$ and $\etaVec(\zVec;\ThetaVec)$ from the standard parameterization to the natural parameters of the corresponding distributions.
Notice that the mapping for the observable distribution, $\etaVec(\zVec;\ThetaVec)$, may additionally depend on the latent variable $\zVec$.
We then assume the following parameterization criterion (compare Definition 2.3 of \citet[][]{LuckeWarnken2024}):

\begin{definition}[Parameterization Criterion]\label{def:Param_Crit}
Consider an EF generative model where the natural parameters of the prior and observable distributions are given by the mappings $\zetaVec(\PsiVec)$ and $\etaVec(\zVec;\ThetaVec)$, respectively.
The EF generative model fulfills the parameterization criterion if the following two properties hold:
\begin{itemize}
  \item[(A)] There exists a function $\alphaVec(\PsiVec)$ depending on the parameters of the prior distribution such that
  \begin{align}
    \zetaVec(\PsiVec) = \frac{\partial \zetaVec(\PsiVec)}{\partial \PsiVecT}\alphaVec(\PsiVec).\label{EqnCondZeta}
  \end{align}
  \item[(B)] There exists a non-empty subset $\thetaVec$ of $\ThetaVec$ and a function $\betaVec(\ThetaVec)$ depending on the parameters of the observable distribution but not depending on the latent variable such that
  \begin{align}
    \etaVec(\zVec;\ThetaVec) = \frac{\partial \etaVec(\zVec;\ThetaVec)}{\partial \thetaVecT} \betaVec(\ThetaVec).\label{EqnCondEta}
  \end{align}
\end{itemize}
Here, $\frac{\partial \zetaVec(\PsiVec)}{\partial \PsiVecT}$ and $\frac{\partial \etaVec(\zVec;\ThetaVec)}{\partial \thetaVecT}$ denote the Jacobians of the mappings to the natural parameters. 
Importantly, for the second condition, we only need to compute the Jacobian with respect to a suitable subset $\thetaVec$ of the obeservable's parameters $\ThetaVec$. This simplifies the computations when the observable distribution is more intricate and depends on many parameters.
\end{definition}

\noindent{}Verifying the parameterization criterion for a given model is the main task. If the criterion is verified, then (\ref{EqnTheoremIntro}) holds for the ELBO at all stationary points. By using the readily available expressions for entropies of the constituting distributions of a given model, we will see that often concise expressions for the ELBO at stationary points can be obtained.
\section{Sigmoid Belief Networks}
\label{SecSBN}
As the first concrete generative model let us consider the Sigmoid Belief Network (SBN; \cite{Neal1992,HintonEtAl2006}). SBNs are amongst the first and most well-known generative models with an extensive body of literature. SBNs are also often considered as prototypical examples for Bayes Nets. Here we study SBNs with one set of latents and one set of observables in their standard form (but we will also briefly discuss an extension to two latent layers).
We state the generative model in the standard notation used in the literature, i.e., we will refer with $\Theta$ to all model parameters (prior and noise model distributions). Only for the proof, we will split the model parameters into one set of parameters for the prior distribution, and one set of parameters for the observable distribution.
%
%
\begin{definition}[Sigmoid Belief Nets]
  \label{def:SBN}
The generative model of an SBN is defined as follows:
\begin{align}
\zVec\ \  &\sim&\hspace{-2ex} p_{\Theta}(\zVec) &= \prod_{h=1}^H\Bern(z_h;\,\pi_h), \ \mbox{with}\ 0<\pi_h<1, \label{EqnSBNA}\\ 
%
%
%
\xVec\ \  &\sim&\hspace{-2ex} p_{\Theta}(\xVec\,|\,\zVec) &= \prod_{d=1}^D \Bern\Big(x_d;\,\Scal\big(\sum_{h=1}^{H} W_{dh}z_h+\mu_d   \big)\Big),\ \mathrm{where}\  \Scal(a)=\frac{1}{1+e^{-a}}\ ,\label{EqnSBNB}
%
%
\end{align}
and where $\Bern(z;\,\pi) = \pi^z (1-\pi)^{1-z}$ is the standard parameterization of the Bernoulli distribution, and the same applies for $\Bern(x;\,\cdot)$ for
the observables. The set of all parameters is $\Theta=(\piVec,W,\muVec)$.
\end{definition}
%
%
\begin{proposition}[Sigmoid Belief Nets]
  \label{prop:SBN}
A Sigmoid Belief Net (Definition~\ref{def:SBN}) is an EF generative model which satisfies the parameterization criterion (Definition~\ref{def:Param_Crit}).
It therefore applies at all stationary points:
\begin{eqnarray}
\FF(\Phi,\Theta) &=&  \frac{1}{N}\sum_{n=1}^N \HH[\qPhiN(\zVec)]  \phantom{ii} \phantom{\sum_{n=1}^N}\hspace{-3ex}-\ \HH[\,p_{\Theta}(\zVec)] \phantom{ii}   -\ \EEE{\;\qBar_{\Phi}}{ \HH[\,p_\Theta(\xVec\,|\,\zVec)] }\,. \phantom{\small{}ix} 
%
%
\label{EqnTheoremSBN}
\end{eqnarray}
\end{proposition}
\begin{proof}
The generative model (\ref{EqnSBNA}) and (\ref{EqnSBNB}) has by definition two different sets of parameters for prior and noise model
$\Theta=(\PsiVec,\ThetaVec)$, which we arrange into column vectors as follows:\vspace{-1ex}
%
%
\begin{align}
\PsiVec=\piVec=\left(\begin{array}{c} \pi_1 \\ \vdots \\ \pi_H \end{array}\right),\ 
\ThetaVec=\left(\begin{array}{c} \wVec_1 \\ \vdots \\ \wVec_H \\ \muVec  \end{array}\right),\ \mbox{with}\ \wVec_h = \left(\begin{array}{c} W_{1h} \\ \vdots \\ W_{Dh} \end{array}\right),
\end{align}
where $\PsiVec$ has $H$ entries and $\ThetaVec$ has $HD+D$ entries (for all weights and offsets).
%
%
\noindent{}The generative model (\ref{EqnSBNA}) and (\ref{EqnSBNB}) is an EF generative model because it can be reparameterized as follows:
\begin{align}
  \zVec &\sim p_{\zetaVec(\piVec)}(\zVec)\,,                          &&\ \mbox{where}&\zeta_h(\piVec) &= \log\big( \frac{\pi_h}{1-\pi_h} \big), \label{EqnSBNC}\\ 
  \xVec &\sim p_{\etaVec(\zVec;\,\ThetaVec)}(\xVec)\,,  &&\ \mbox{where} &\etaVec(\zVec;\,\ThetaVec) &=\disT \sum_h \wVec_h\,z_h\,+\,\muVec\,,
%
\end{align}
and where the distributions $p_{\zetaVec}(\zVec)$ and $p_{\etaVec}(\xVec)$ are now Bernoulli distributions in the exponential family form.
For the parameterization criterion, the Jacobian matrices are given by:
\begin{align}
  \Inew   &=    \left(\begin{array}{ccc}  \pi_1^{-1}(1-\pi_1)^{-1} &    &      \\
                                                                                                                  &    \ddots           &               \\
                                                                                                                  &                     &  \pi_H^{-1}(1-\pi_H)^{-1} \end{array}\right),\\ 
  \Jnew   &=    \left(z_1\OneDD, \dots, z_H \OneDD, \OneDD \right),
\end{align}
where we have chosen all parameters for $\thetaVec$ (i.e., $\thetaVec=\ThetaVec$) for $\Jnew$.
As all $\pi_h\in(0,1)$, the diagonal entries of the $(H\times{}H)$-matrix $\Inew$ are non-zero such that $\Inew$ is invertible.
Part~A of the criterion is therefore fulfilled by choosing $\vec{\alpha}(\PsiVec)=\Bigl(\Inew\Bigr)^{-1}\zetaVec(\PsiVec)$. For part B, it is straightforward to see that
\begin{align}
  \betaVec(\ThetaVec)=\left(\begin{array}{c}
    \vec{w}_1\\
    \vdots\\
    \vec{w}_H\\
    \vec{\mu}
  \end{array}\right)
\end{align}
%
%
%
%
%
%
%
%
%
%
%
%
%
fulfills Eqn.~\ref{EqnCondEta}. Hence, part~B of the parameterization criterion is also satisfied. Consequently, Eqn.~\ref{EqnTheoremIntro} applies due to Theorem~1 of \cite{LuckeWarnken2024}, and we obtain:
%
\begin{eqnarray}
\disT\FF(\Phi,\PsiVec,\ThetaVec) &=&\disT  \frac{1}{N}\sum_{n=1}^N \HH[\qPhiN(\zVec)]  \ -\ \HH[\,\pPsi(\zVec)]  -\ \EEE{\qBar_{\Phi}}{ \HH[\,\pT(\xVec\,|\,\zVec)] }\,. \phantom{\small{}ix} 
\label{EqnTheoremProofSBN}
\end{eqnarray}
%
%
\end{proof}
\noindent{}For Sigmoid Belief Nets it is also common to consider more than one set of latent variables. 
For instance, we could add another hidden layer. Such a Sigmoid Belief Net is then defined by: 
\begin{align}
  p_{\Theta}(\zVec) &= \prod_{h=1}^{H}\Bern(z_h;\pi_h),\ \mbox{with}\ 0<\pi_h<1,\\
  %
  p_{\Theta}(\yVec \mid \zVec) &= \prod_{i=1}^{I}\Bern\Big(y_i;\Scal\big(\sum_{h=1}^{H}M_{ih}z_{h}+\nu_{i}\big)\Big),\\ 
  %
  p_{\Theta}(\xVec\mid \yVec) &= \prod_{d=1}^D \Bern\Big(x_d; \Scal\big(\sum_{i=1}^I W_{di} y_i +\mu_d\big)\Big),
\end{align}
where $\Scal$ is again the function $\Scal(a)=\frac{1}{1+\mathrm{e}^{-a}}$ and $\Bern(z,\pi)=\pi^z(1-\pi)^{1-z}$ the standard parameterization of the Bernoulli distribution. 
Considering the proof of Theorem~1 in \cite{LuckeWarnken2024}, adding another layer would simply add another term to the ELBO (\ref{EqnFFIntro}). Equality to the corresponding entropy of that term can
then be shown precisely as in the proof of Theorem~1 in \cite{LuckeWarnken2024} after verifying the property of Definition~\ref{def:Param_Crit} in analogy to Prop.\,\ref{prop:SBN}. Therefore, also for a three layer SBN the ELBO becomes equal to an
entropy sum. The equality is (without giving an explicit proof) given by:
\begin{align}
    \FF(\Phi,\Theta) &= \frac{1}{N}\sum_{n=1}^N \HH[q^{(n)}_{\Phi}(\zVec,\yVec)] - \HH[p_\Theta(\zVec)] - \mathbb{E}_{\;\qBar_\Phi}\{\HH[\,p_\Theta(\yVec\,|\, \zVec)]\}\\
    &\hspace{20mm}- \mathbb{E}_{\;\qBar_{\Phi}}\{\HH[\,p_{\Theta}(\xVec\mid \yVec)]\}\,.\nonumber
\end{align}
%
So the additional layer simply adds the (expected) entropy of the added layer. 
%
%
%
%
\section{Generative Models with Gaussian Observables and Probabilistic PCA}
\label{SecPCA}
%
%
While Eqn.~\ref{EqnTheoremIntro} applies for a broad range of generative models, those models with Gaussian observable distributions 
are very common and, therefore, represent an important subclass of generative models. At the same time, the properties of Gaussian distributions
further simplify the result in Eqn.~\ref{EqnTheoremIntro}. For these reasons, we will treat Gaussian observables explicitly in the following.
\begin{definition}[Gaussian observables with scalar variance]
    \label{def:gaussian_obs_scalar}
We consider generative models as in Eqn.~\ref{EqnGenModelIntro} where the observable distribution $\pT(\xVec\,|\,\zVec)$ is a Gaussian. Concretely,
we consider the following class of generative models:
%
%
\begin{eqnarray}
 \zVec &\sim& \pPsi(\zVec)\,, \\ 
 \xVec &\sim& \pT(\xVec\,|\,\zVec) = \NN(\xVec; \muVec(\zVec;\wVec),\sigma^2 \One),\ \mbox{where }\ \sigma^2>0,  \vspace{-3ex}
\end{eqnarray}
and where $\muVec(\zVec;\wVec)$ is a well-behaved function (with parameters~$\wVec$) from the latents $\zVec$ to the mean of the Gaussian.
%
\end{definition}
\noindent{}For linear functions $\muVec$, the vector $\wVec$ can be thought of as containing all entries of a weight matrix $W$ and potentially all offsets.
$\muVec(\zVec;\wVec)$ can, however, also represent a non-linear function defined using a potentially intricate deep neural network (DNN). In such a non-linear
case, the class of functions defined by Definition~\ref{def:gaussian_obs_scalar} contains standard variational autoencoders (VAEs; \cite{KingmaWelling2014}, and many more), VAEs with sparse prior (e.g. \cite{ConnorEtAl2021,DrefsEtAl2023}), and other VAEs with non-standard (but exponential family) priors.\\

\noindent{}As outlined in the introduction, equality of the ELBO to entropy sums, as stated in Eqn.~\ref{EqnTheoremIntro}, can be shown by applying Theorem~1 from \cite{LuckeWarnken2024}.
To utilize this theorem, the parameterization criterion presented in Definition~\ref{def:Param_Crit} has to be fulfilled.
However, the theorem requires an additional assumption that we have not addressed so far, as it was always satisfied for SBNs.
To apply Theorem 1, both the prior and the observable distributions must have constant base measures.
This condition is trivially fulfilled here for the observable distribution, since it is Gaussian distribution, which always has a constant base measure.
For the prior distribution, we must explicitly impose this condition.
However, by applying Theorem~2 from \cite{LuckeWarnken2024} instead of Theorem~1, one could generalize the following proposition without assuming a constant base measure.
This would lead to a similar, but slightly more intricate, result.
%


%
%
\begin{proposition}[Gaussian observables with scalar variance]
    \label{prop:Gauss_obs_scalar}
Consider the generative model of Definition~\ref{def:gaussian_obs_scalar}. If the prior $\pPsi(\zVec)=p_{\zetaVec(\PsiVec)}(\zVec)$
is an exponential family distribution with constant base measure that satisfies part~A of the parameterization criterion (Definition~\ref{def:Param_Crit}), then at all stationary points apply:
%
%
\begin{eqnarray}
\FF(\Phi,\PsiVec,\ThetaVec) &=&  \frac{1}{N}\sum_{n=1}^N \HH[\qPhiN(\zVec)]  \phantom{ii} \phantom{\sum_{n=1}^N}\hspace{-3ex}-\ \HH[\,\pPsi(\zVec)] \phantom{ii}   -\ \HH[\,\pT(\xVec\,|\,\zVec)]  \phantom{\small{}ix} \nonumber\\
                    &=&  \frac{1}{N}\sum_{n=1}^N \HH[\qPhiN(\zVec)]  \phantom{ii} \phantom{\sum_{n=1}^N}\hspace{-3ex}-\ \HH[\,\pPsi(\zVec)] \phantom{ii}   
                    -\ \frac{D}{2} \log\big(2\pi{}e\sigma^2\big)\,.  \phantom{\small{}ix} 
\label{EqnPropGaussScalar}
%
%
%
%
%
%
\end{eqnarray}
\end{proposition}
\begin{proof}
%
%
%
%
%
\noindent{}We need to show that part B of the parameterization criterion is fulfilled. We abbreviate the variance $\sigma^2$ by $\sit$, and we
choose $\thetaVec=\sit$ as subset of $\ThetaVec$ (i.e., we chose a single-valued vector~$\thetaVec$).
The Jacobian $\Jnew$ in Eqn.~\ref{EqnCondEta} is then a column vector.
%
%
%
Now consider the natural parameters of a Gaussian with scalar variance, $\etaVec=\left(\begin{array}{c} \frac{\vec{\mu}}{\sit} \\ -\frac{1}{2\sit}  \end{array}\right)$.
The form of the natural parameters means that we can rewrite $\etaVec(\zVec;\ThetaVec)\,=\,\frac{1}{\sit}\,\tilde{\etaVec}(\zVec;\wVec)$,
where the second factor does not depend on $\sit$. The Jacobian $\Jnew$ now becomes:
\begin{eqnarray}
\Jnew =\del{\sit}\ \etaVec{(\zVec;\ThetaVec)} =\big( \del{\sit} \frac{1}{\sit} \big)\  \tilde{\etaVec}(\zVec;\wVec)
= -\frac{1}{\sit}\,\frac{1}{\sit}\ \tilde{\etaVec}(\zVec;\wVec)=-\frac{1}{\sit}\  \etaVec{(\zVec;\ThetaVec)}\,.
\end{eqnarray}
Therefore, the scalar function $\beta(\tilde{\sigma})=-\tilde{\sigma}$ satisfies Eqn.\,\ref{EqnCondEta}, which shows that part~B of the parameterization criterion is fulfilled.
%
%
%
%
%
%
%
Consequently, Theorem~1 in \cite{LuckeWarnken2024}, and thus Eqn.~\ref{EqnTheoremIntro}, applies for the generative model.

The entropy of the Gaussian $\pT(\xVec\,|\,\zVec)$ is given by $\frac{D}{2} \log\big(2\pi{}e\sit\big)$, i.e., it does not depend on $\zVec$. The last term of~(\ref{EqnTheoremIntro}) consequently
simplifies to:
\begin{eqnarray}
%
%
\EEE{\qBar_{\Phi}}{ \HH[\,\pT(\xVec\,|\,\zVec)] }\,=\,\frac{D}{2} \log\big(2\pi{}e\sit\big) \,=\,\HH[\,\pT(\xVec\,|\,\zVec)]\,.
\end{eqnarray}
\end{proof}
\noindent{}The maybe most prominent elementary generative model covered by Proposition~\ref{prop:Gauss_obs_scalar} is probabilistic PCA (\mbox{p-PCA;} \cite{Roweis1998,TippingBishop1999}). P-PCA has a Gaussian prior and a Gaussian noise model, i.e., both distributions are in the exponential family.
Given Proposition~\ref{prop:Gauss_obs_scalar}, we only have to show that the parameterization criterion applies for the prior distribution. This can be shown similarly as was done for the noise distribution above. However, p-PCA first has to be parameterized similarly to Example~2 in \cite{LuckeWarnken2024} as we require a parameterized prior. In the standard parameterization of p-PCA (i.e., $p_{\Theta}(\zVec)=\NN(\zVec;0,\One)$ and $p_{\Theta}(\xVec\,|\,\zVec)=\NN(\xVec;W\zVec+\muVec,\sigma^2\One)$) and for the case when full posteriors are used as variational distributions, we then obtain:
%
\begin{eqnarray}
\LL(\Theta) &=&  \HH[ \,p_{\Theta}(\zVec\,|\,\xVec)]  \,-\, \HH[\,p_{\Theta}(\zVec)] \,-\, \HH[\,p_{\Theta}(\xVec\,|\,\zVec)] \\[1ex]
                        &=& \disS \frac{1}{2} \log \big( (2\pi\mathrm{e})^H\det( \frac{1}{\sigma^2} W^{\mathrm{T}}W \,+\, \One)\,\big)\,-\,\frac{H}{2}\log(2\pi\mathrm{e})-\, \frac{D}{2} \log\big(2\pi e \sigma^2 \big)\\
%
%
                        &=&\disS \frac{1}{2} \log \big( \det( \frac{1}{\sigma^2} W^{\mathrm{T}}W \,+\, \One)\,\big) \,-\, \frac{D}{2} \log\big(2\pi e \sigma^2 \big)\,,
\label{EqnPropPCA}\vspace{1ex}
\end{eqnarray}
at all stationary points.
%
%
The result can be shown to be consistent with the maximum likelihood solution of \cite{TippingBishop1999} by expressing $M=W^T W + \sigma^2 \One$ and
$\sigma^2$ in terms of eigenvalues of the data covariance matrix. Also note the relation to variational autoencoders with linear decoder as discussed
in \cite{DammEtAl2023}. 
%


While Definition~\ref{def:gaussian_obs_scalar} already covers many well-known models including probabilistic PCA, other models such as Factor Analysis require a further generalization. Concretely, we consider generative models with observables distributed according to a Gaussian with diagonal covariance matrices.
%
\begin{definition}[Gaussian observables with diagonal covariance]
    \label{def:gaussian_obs_diag}
We consider generative models as in Eqn.~\ref{EqnGenModelIntro} where the observable distribution $\pT(\xVec\,|\,\zVec)$ is a Gaussian with
diagonal covariance matrix:
%
%
\begin{eqnarray}
 \zVec &\sim& \pPsi(\zVec)\,, \\ 
 \hspace{10mm}\xVec &\sim& \pT(\xVec\,|\,\zVec) = \NN(\xVec; \muVec(\zVec;\wVec),\Sigma),\ \mbox{where }\ \Sigma=\diag(\sigma_1^2,...,\sigma_D^2)
\,,\     \sigma_d^2>0,  \vspace{-3ex}\label{EqnNoiseGaussDiag}
\end{eqnarray}
where $\muVec(\zVec;\wVec)$ is again a well-behaved function (with parameters $\wVec$) from the latents $\zVec$ to the mean of the Gaussian.

\end{definition}
%
%
\begin{proposition}[Gaussian observables with diagonal covariance]
  \label{prop:Gauss_obs_diag}
Consider the generative model of Definition~\ref{def:gaussian_obs_diag}. If the prior $\pPsi(\zVec)=p_{\zetaVec(\PsiVec)}(\zVec)$
is an exponential family distribution with constant base measure that satisfies part~A of the parameterization criterion (see Definition~\ref{def:Param_Crit}), then at all stationary points apply:
%
%
%
\begin{eqnarray}
\FF(\Phi,\PsiVec,\ThetaVec) &=&  \frac{1}{N}\sum_{n=1}^N \HH[\qPhiN(\zVec)]  \phantom{ii} \phantom{\sum_{n=1}^N}\hspace{-3ex}-\ \HH[\,\pPsi(\zVec)] \phantom{ii}   -\ \HH[\,\pT(\xVec\,|\,\zVec)]  \phantom{\small{}ix} \nonumber\\
                  &=&  \frac{1}{N}\sum_{n=1}^N \HH[\qPhiN(\zVec)]  \phantom{ii} \phantom{\sum_{n=1}^N}\hspace{-3ex}-\ \HH[\,\pPsi(\zVec)] \phantom{ii}   
                  -\ \sum_{d=1}^{D} \frac{1}{2} \log\big(2\pi{}e\sigma_d^2\big)\,.  \phantom{\small{}ix} 
\label{EqnPropGaussDiag}
%
%
%
%
%
%
\end{eqnarray}
\end{proposition}
\begin{proof}
%
%
To show that part B the parameterization criterion is fulfilled, we consider as subset $\thetaVec$ of $\ThetaVec$ the vector of all variances $\sigma^2_d$ (with abbreviation $\sit_d=\sigma^2_d$) and the following natural parameters:
\begin{align}
\thetaVec=
\begin{pmatrix}  \sit_1 \\ \vdots  \\ \sit_D  \end{pmatrix},\
\etaVec(\zVec;\ThetaVec)=
\begin{pmatrix}  \aVec(\zVec;\ThetaVec) \\[1ex] \bVec(\zVec;\ThetaVec)  \end{pmatrix},\ 
\aVec(\zVec;\ThetaVec) = 
\begin{pmatrix}  \frac{\mu_1(\zVec;\wVec)}{\sit_1} \\[1ex] \vdots \\[1ex] \frac{\mu_D(\zVec;\wVec)}{\sit_D} \end{pmatrix},\ 
\bVec(\zVec;\ThetaVec) = 
\begin{pmatrix}  -\frac{1}{2\,\sit_1} \\[1ex] \vdots \\[1ex] -\frac{1}{2\,\sit_D} \end{pmatrix}. 
\end{align}
Given $\thetaVec$ as above, the Jacobian in~(\ref{EqnCondEta}) is a $(2D\times{}D)$-matrix given as follows:
\begin{eqnarray}
\Jnew \ =\ \left(\begin{array}{c}
  \displaystyle\delll{\vec{a}(\zVec;\ThetaVec)}{\thetaVecT}\\[0.8em]
  \displaystyle\delll{\vec{b}(\zVec;\ThetaVec)}{\thetaVecT}
\end{array}\right)
= \left(\begin{array}{c}
  A\\
  B
\end{array}\right),
\end{eqnarray}
where $A$ and $B$ are $(D\times{}D)$-diagonal matrices with entries:
\begin{align}
A = \left(\begin{array}{ccc}
  -\frac{\mu_1(\zVec;\wVec)}{\tilde{\sigma}_1^2} &  & \\
   & \ddots & \\
    & & -\frac{\mu_D(\zVec;\wVec)}{\tilde{\sigma}_D^2}
\end{array}\right)\ \mbox{\ and\ }\ 
B = \left(\begin{array}{ccc}
  \frac{1}{2\tilde{\sigma}_1^2} & & \\
   & \ddots & \\
   & & \frac{1}{2\tilde{\sigma}_D^2}
\end{array}\right).
\end{align}
By choosing $\betaVec(\thetaVec)=-\left(\tilde{\sigma}_1,\dots,\tilde{\sigma}_D\right)^{\mathrm{T}}$, Eqn.~\ref{EqnCondEta} und hence the second part of the parameterization criterion is fullfilled and Theorem~1 in \cite{LuckeWarnken2024}, and thus Eqn.~\ref{EqnTheoremIntro}, applies.

The entropy of the Gaussian $\pT(\xVec\,|\,\zVec)$ is finally given by $\frac{1}{2}\sum_d \log\big(2\pi{}e\sit_d\big)$, i.e., it again does not depend on $\zVec$. The last term of (\ref{EqnTheoremIntro}) consequently simplifies to:
\begin{eqnarray}
%
%
\EEE{\qBar_{\Phi}}{ \HH[\,\pT(\xVec\,|\,\zVec)] }\,=\,\frac{1}{2}\sum_d \log\big(2\pi{}e\sit_d\big)\,=\,\HH[\,\pT(\xVec\,|\,\zVec)]\,.
\end{eqnarray}
\end{proof}

\noindent{}Notably, one has to consider the cases with scalar variance and diagonal covariance separately because scalar variance can in our setting not be treated as a special case of diagonal covariance. Using Proposition~\ref{prop:Gauss_obs_diag}, we could now treat factor analysis (e.g. \cite{Everitt1984,BartholomewKnott1987}) as a model with Gaussian prior and a linear mapping to observables.


\section{Mixture Models}
\label{SecMixtures}
\begin{sloppypar}
As the last example we consider mixture models with exponential family distributions as noise distributions. Such models are very
closely related to Bregman clustering (e.g.\,\cite{McLachlanPeel2004};\,\cite{BanerjeeEtAl2005}), and mixtures are a tool that is for machine learning and data science of similar importance as PCA or Factor Analysis. In this section we follow the standard notational convention for mixture models and replace the latent variable $\zVec$ by $c$, where $c$ takes on integer values.\\
\end{sloppypar}

\noindent{}As previously mentioned, Eqn.~\ref{EqnTheoremIntro} follows directly from Theorem 1 in \cite{LuckeWarnken2024}.
However, the application of this theorem requires not only the parameterization criterion (Definition~\ref{def:Param_Crit}) but also the assumption of constant base measures for both the prior and observable distributions.
In contrast, when constant base measures are not assumed, Theorem 2 from \cite{LuckeWarnken2024} can be used instead of Theorem 1.
This results in a slightly more intricate entropy expression than that given in Eqn.~\ref{EqnTheoremIntro}, as it involves replacing the entropies with a term referred to as pseudo entropy in \cite{LuckeWarnken2024}.
For clarity, we here distinguish between the two cases of constant and potentially non-constant base measures.
Since many machine learning models commonly assume constant base measures, the results corresponding to this assumption are typically sufficient.
When the assumption of constant base measures is not fulfilled, the reader can refer to the more general results for potentially non-constant base measures.
%
\begin{definition}[EF Mixture with Constant Base Measure]
    \label{def:mixtures}
We consider generative models with a latent variable $c$ and observables $\xVec$. The latent can take on one of $C$ different values, i.e.\ $c\in\{1,\ldots,C\}$, each with probability $\pi_c>0$ where $\sum_{c=1}^C \pi_c = 1$. The observable $\xVec$
is distributed according to a distribution $p(\xVec\,|\,\ThetaVec_c)$ with parameters depending on the value of the latent $c$.
More formally, the generative model is given by:
%
%
\begin{align}
%
%
 c &\,\sim\, \mathrm{Cat}(c\,;\vec{\pi}), \label{EFMixPrior}\\ 
 \xVec &\,\sim\, p(\xVec\,|\,\ThetaVec_c), \label{EFMixNoise} \vspace{-3ex}
\end{align}
where $\mathrm{Cat}(c\,;\vec{\pi})$ is the categorical distribution with standard parameterization.
The set of all model parameters $\Theta$ is given by $\Theta=\big(\piVec,\ThetaVec_1,\ThetaVec_2,\ldots,\ThetaVec_C\big)$.

We refer to the model (\ref{EFMixPrior}) and (\ref{EFMixNoise}) as an {\em exponential family mixture (EF mixture) with constant base measure} if
the observable distribution $p(\xVec\,|\,\ThetaVec_c)$ can be written as an exponential family distribution
$p_{\etaVec(\ThetaVec_c)}(\xVec)$ with a constant base measure, where $\etaVec(\cdot)$ maps an $L$-dim vector $\ThetaVec_c$ of standard parameters
to an $L$-dim vector of natural parameters. 
%
%

\end{definition}

\noindent{}For a mixture of scalar Gaussian distributions, $\ThetaVec_c$ would be an $L=2$ dimensional vector equal to $\left(\begin{array}{c} \mu_c \\[0ex] \sigma_c^2  \end{array}\right)$, where $\mu_c$ and~$\sigma_c^2$ are, respectively, mean and variance of cluster $c$. The mapping to natural parameters $\etaVec(\cdot)$ would be given by $\etaVec(\ThetaVec_c) = \left(\begin{array}{c} \frac{\mu_c}{\sigma^2_c} \\[0.5ex] -\frac{1}{2\sigma_c^2}  \end{array}\right)$.
%
For a mixture of gamma distributions, $\ThetaVec_c$ would be an $L=2$ dimensional vector equal to $\left(\begin{array}{c} \alpha_c \\[0ex] \beta_c  \end{array}\right)$, where $\alpha_c$ and $\beta_c$ are the standard parameters (shape and rate, respectively). The mapping to natural parameters $\etaVec(\cdot)$ would be given by 
$\etaVec(\ThetaVec_c)= \left(\begin{array}{c} \alpha_c - 1 \\[0ex] - \beta_c  \end{array}\right)$. Similarly for other standard distributions. 
%
%
For the proof of the proposition on EF mixtures with constant base measure below, we will again use vectors for the parameters and will distinguish between prior parameters $\PsiVec=(\pi_1,\ldots,\pi_{C-1})^{\mathrm{T}}$ and noise model parameters $\ThetaVec=(\ThetaVec_1,\ldots,\ThetaVec_C)^{\mathrm{T}}$. Note that we will only consider $(C-1)$ prior parameters as $\pi_C$ can be taken to be determined by $\pi_1$ to $\pi_{C-1}$. 
%

We remark the technical but important difference between the mapping $\etaVec(\cdot)$ of Definition~\ref{def:mixtures} and the mapping $\etaVec(c;\ThetaVec)$ which takes
the role of $\etaVec(\zVec;\ThetaVec)$ as used for Definition~\ref{def:Param_Crit}. $\etaVec(\cdot)$ is the usual mapping from standard parameters of an exponential family distribution to its natural parameters (see the just discussed examples Gaussian and gamma distributions). $\etaVec(c;\ThetaVec)$ also maps standard parameters to natural parameters of the observable distributions (hence the same symbol) but, in contrast to $\etaVec(\cdot)$, the mapping $\etaVec(c;\ThetaVec)$ has the vector of {\em all} noise model parameters as input. The Jacobians of the mappings are consequently different. The Jacobian of $\etaVec(c;\ThetaVec)$ (computed w.r.t.\ all parameters $\ThetaVec$) is an $(L\times{}(CL))$-matrix, while the Jacobian of $\etaVec(\cdot)$ is a square \mbox{$(L\times{}L)$-matrix}. This smaller $(L\times{}L)$-Jacobian is important to state a sufficient condition for the following proposition. The large Jacobian of $\etaVec(c;\ThetaVec)$ is important for the proposition's proof.

\begin{proposition}[Exponential Family Mixtures with Constant Base Measure]
  \label{prop:mixtures}
%
%
%
Consider an EF mixture model of Definition~\ref{def:mixtures} where $\etaVec(\cdot)$ maps the standard parameters of the noise model to the distribution's natural parameters. If the Jacobian of $\etaVec(\cdot)$ is everywhere invertible, then the model is an EF generative model (see Section~\ref{SecIntro}) and the parameterization criterion (Definition~\ref{def:Param_Crit}) is fulfilled. It then applies at all stationary points:
\begin{eqnarray}
%
%
\hspace{8mm}\FF(\Phi,\Theta) &=&  \frac{1}{N}\sum_{n=1}^N \HH[\qPhiN(c)]  \phantom{i} \phantom{\sum_{n=1}^N}\hspace{-3ex}-\ \HH[\,\mathrm{Cat}(c\,;\vec{\pi})\,] \phantom{i}   -\ \EEE{\;\qBar_{\Phi}(c)}{ \HH[\,p(\xVec\,|\,\ThetaVec_c)] }\,, \phantom{\small{}ix} 
%
%
%
%
\label{EqnPropMixtures}
%
%
%
%
%
%
\end{eqnarray}
where we explicitly denoted that the aggregate posterior $\qBar_\Phi(c)=\frac{1}{N}\sum_n \qn_\Phi(c)$ depends on $c$.
\end{proposition}
\begin{proof}
We change to vectorial notation for the parameters as we did for the previous proofs. With $\PsiVec=(\pi_1,\ldots,\pi_{C-1})^{\mathrm{T}}$ we denote the prior parameters, and with $\ThetaVec=(\ThetaVec_1,\ldots,\ThetaVec_C)^{\mathrm{T}}$ we denote the noise model parameters.
%
%
%
%
%
%
The categorical distribution of the prior can be rewritten in exponential family form as follows:
\begin{eqnarray}
p_{\zetaVec(\PsiVec)}(c) \hspace{-1ex}&=&\hspace{-1ex} h(c) \exp\big( \zetaVecT(\PsiVec)\, \TVec(c) \,-\, A(\PsiVec)  \big)\,,
%
%
%
\label{EqnExpCat}
\end{eqnarray}
where the natural parameters $\zetaVec(\PsiVec)$ have length $(C-1)$ and are defined for $i=1,\ldots,(C-1)$ by $\zeta_i(\PsiVec)= \log\big( \frac{\pi_i}{ 1 - \sum_{i'=1}^{C-1}\pi_{i'} } \big)$. The sufficient statistics $\TVec(c)$ is also of length $(C-1)$ and are defined for $i=1,\ldots,(C-1)$ by 
$T_i(c) = \delta_{ci}$ (i.e., $T_i(c)=1$ if $c=i$ and zero otherwise). The partition function is
$A(\PsiVec)=-\log\big( 1 - \sum_{c=1}^{C-1}\pi_c \big)$ and the base measure is a constant, $h(c)=1$. 
For the observable distribution, we assumed in Definition~\ref{def:mixtures} that it can be written as an exponential family distribution, $p_{\etaVec(\ThetaVec_c)}(\xVec)$, with constant base measure. With this assumption, the mixture model of Definition~\ref{def:mixtures} is an EF generative model (see Section~\ref{SecIntro}). $\zetaVec(\PsiVec)$ as defined above are the natural parameters of the prior; and $\etaVec(c;\,\ThetaVec)=\etaVec(\ThetaVec_c)$ are the natural parameters of the observable distribution (for mixtures we use $c$ for the latent instead of $\zVec$ for previous models).
It remains to be shown that the parameterization criterion holds (Definition~\ref{def:Param_Crit}).


Regarding part A of the parameterization criterion, the Jacobian matrix $\Inew$ of $\zetaVec(\PsiVec)$ is a $(C-1)\times(C-1)$-matrix. The elements of the Jacobian can be derived to be:
\begin{eqnarray}
\Biggl(\Inew\Biggr)_{ic} \ =\  \del{\pi_c} \zeta_i(\PsiVec) \ =\ \del{\pi_c} \log\big( \frac{\pi_i}{ 1 - \sum_{i'=1}^{C-1}\pi_{i'} } \big) \ =\ \delta_{ic}\,\frac{1}{\pi_c}\,+\,\frac{1}{\pi_C},
%
\end{eqnarray}
where we abbreviated $\pi_C=1-\sum_{c=1}^{C-1}\pi_{c}$. We now consider the $i$-th component of the right-hand side of (\ref{EqnCondZeta}) which is given by:
\begin{eqnarray}\label{EqnCategorialCond}
\hspace{6mm}\Biggl(\Inew\alphaVec(\PsiVec)\Biggr)_i = \sum_c \Bigl(\Inew\Bigr)_{ic} \alpha_c(\PsiVec) = \frac{1}{\pi_i}\alpha_i(\PsiVec) +\frac{1}{\pi_C}\sum_c\alpha_c(\PsiVec) \overset{!}{=} \zeta_i(\PsiVec).
\end{eqnarray}
By multiplying with $\pi_i$ and summing over all components $i=1,\ldots, C-1$, we can conclude from (\ref{EqnCategorialCond}) that:
\begin{align}
&\sum_i \pi_i \zeta_i(\PsiVec) = \sum_i \alpha_i(\PsiVec) + \frac{1}{\pi_C}\Big(\sum_{i}\pi_i\Big) \Big(\sum_c \alpha_c(\PsiVec)\Big) = \frac{1}{\pi_C}\sum_i \alpha_i(\PsiVec)\\
&\Rightarrow \frac{1}{\pi_C} \sum_i \alpha_i(\Psi)= \sum_i \pi_i\zeta_i(\PsiVec)=:\rho(\PsiVec),\label{EqnCategorialSum}
\end{align}
where we introduced the function $\rho(\PsiVec)$ as an abbreviation. By using $c$ as the summation index and substituting (\ref{EqnCategorialSum}) into (\ref{EqnCategorialCond}), we get
\begin{align}
\alpha_i(\PsiVec) = \pi_i\big(\zeta_i(\PsiVec)- \rho(\PsiVec)\big),\ \ \ i=1,\ldots,C-1.
\end{align}
The corresponding $\alphaVec(\PsiVec)$ satisfies Eqn.~\ref{EqnCondZeta}. Therefore, Part~A of the parameterization criterion is fulfilled.

To construct the Jacobian $\delll{\etaVec(c;\ThetaVec)}{\thetaVecT}$ for part B of the criterion, we use all parameters $\ThetaVec$, i.e.\ $\thetaVec=\ThetaVec$.
The Jacobian is then the ($L\times{}CL$)-matrix given by:
\begin{align}
\delll{\etaVec(c;\ThetaVec)}{\thetaVecT} = \left(\delll{\etaVec(c;\ThetaVec)}{\ThetaVecT_1},\dots,\delll{\etaVec(c;\ThetaVec)}{\ThetaVecT_C}\right)
=\left(\ZeroLL,\dots, \ZeroLL,\delll{\etaVec(\ThetaVec_c)}{\ThetaVecT_c} , \ZeroLL,\dots,\ZeroLL\right).
\end{align}
Here $\ZeroLL$ denotes the ($L\times L$)-matrix where all entries are zero. The only non-vanishing components are those with derivative w.r.t. the parameters $\ThetaVec_c$. According to our assumptions, the Jacobian $\delll{\etaVec(\ThetaVec_c)}{\ThetaVecT_c}$ is invertible. Therefore, the vector valued function
\begin{eqnarray}
\betaVec(\ThetaVec) = \left(\Bigl(\delll{\etaVec(\ThetaVec_1)}{\ThetaVecT_1}\Bigr)^{-1}\etaVec(\ThetaVec_1),\dots,\Bigl(\delll{\etaVec(\ThetaVec_C)}{\ThetaVecT_C}\Bigr)^{-1} \etaVec(\ThetaVec_C)\right)^{\mathrm{T}} \label{EqnLastBeta}
\end{eqnarray}
is well-defined and independent of $c$. Using (\ref{EqnLastBeta}) we obtain:
\begin{eqnarray}
%
%
\delll{\etaVec(c;\ThetaVec)}{\thetaVecT} \ \betaVec(\ThetaVec) 
&=& \left(\ZeroLL,\dots, \ZeroLL,\delll{\etaVec(\ThetaVec_c)}{\ThetaVecT_c} , \ZeroLL,\dots,\ZeroLL\right) \ \betaVec(\ThetaVec) \nonumber\\
&=& \etaVec(\ThetaVec_c)\ =\ \etaVec(c;\ThetaVec)\,.
\end{eqnarray}
Hence part~B of the parameterization criterion (Eqn.\,\ref{EqnCondEta}) is also fulfilled.

As the model of Definition~\ref{def:mixtures} with the assumption of Prop.~\ref{prop:mixtures} is an EF generative model that fulfills the parameterization criterion, Theorem~1 from \cite{LuckeWarnken2024} applies and therefore Eqn.~\ref{EqnTheoremIntro} holds. Consequently, Eqn.~\ref{EqnPropMixtures} follows.
\end{proof}
\noindent{}The assumptions made in Proposition\,\ref{prop:mixtures} are sufficient for (\ref{EqnPropMixtures}) to apply. The assumptions are also sufficiently broad to apply for most mixtures but they can presumably be weakened further. 
%
For a given mixture model, Proposition\,\ref{prop:mixtures} is relatively easy to apply because the mapping from standard parameterization to natural parameters is usually known and known to be invertible.\\

%
%
%
\noindent{}We now present an entropy result for mixture models in the case when the assumption of a constant base measure is not fulfilled, i.e., when there is no parameterization of the observable distribution that has a constant base measure.
A prominent example of such a model is the Poisson mixture model, which we treat in Section~\ref{SecPMM}.
To derive the corresponding result, we apply Theorem 2 from \cite{LuckeWarnken2024} instead of Theorem 1.
However, \cite{LuckeWarnken2024} introduced for Theorem~2 an alternative ELBO objective, denoted by $\tilde{\mathcal{F}}(\Phi,\Theta)$, which differs from the original ELBO $\FF(\Phi,\Theta)$ by an additive constant.
More precisely, $\tilde{\mathcal{F}}(\Phi,\Theta)$ is given by (see Appendix~B of \cite{LuckeWarnken2024})
%
\begin{align}
  \tilde{\mathcal{F}}(\Phi,\Theta) = \FF(\Phi,\Theta) - \frac{1}{N}\sum_{n=1}^{N} \log\big(h(\xVec^{\,(n)})\big),\label{EqnPseudoELBO}
\end{align}
where $h(\xVec)$ represents the base measure of the observable distribution.
Since the base measure does not depend on any parameters, the optimization of $\tilde{\mathcal{F}}(\Phi,\Theta)$ is entirely equivalent to the optimization of the original ELBO $\FF(\Phi,\Theta)$.
For the new objective $\tilde{\mathcal{F}}(\Phi,\Theta)$, \cite{LuckeWarnken2024} showed an entropy decomposition at stationary points similar to the one given in Eqn.~\ref{EqnTheoremIntro}.
Specifically, the decomposition in this case reads (see their Theorem~2)
\begin{align}
  \ \tilde{\FF}(\Phi,\Theta) =\frac{1}{N}\sum_{n} \HH[\qPhiN(\zVec)]  \,-\, \HH[\,p_{\Theta}(\zVec)] \,-\, \frac{1}{N} \sum_{n} \EEE{\qn_{\Phi}}{ \tilde{\HH}[\,p_{\Theta}(\xVec\,|\,\zVec)] }\,.\phantom{\small{}ix}
\end{align}
\begin{sloppypar}
\noindent{}Here, the quantity $\tilde{\mathcal{H}}[\,p_{\Theta}(\xVec\,|\,\zVec)]$ is referred to as the pseudo entropy, as introduced by \cite{LuckeWarnken2024}.
For the definition of this pseudo entropy, consider an exponential family distribution of the form $p_{\etaVec}(\xVec) = h(\xVec) \exp\big(\etaVecT \TVec(\xVec)-A(\etaVec)\big)$ with base measure $h(\xVec)$.
Moreover, consider the corresponding non-normalized function $\tilde{p}_{\etaVec}(\xVec)=\exp\big(\etaVecT \TVec(\xVec)-A(\etaVec)\big)$ where the base measure is set to one.
Then, the pseudo entropy of $p_{\etaVec}(\xVec)$ is defined to be
\end{sloppypar}
\begin{align}
  \tilde{\HH}[p_{\etaVec}(\xVec)] = -\int p_{\etaVec}(\xVec) \log\big(\tilde{p}_{\etaVec}(\xVec)\big) \mathrm{d}\xVec = -\etaVecT \nabla_{\etaVec}A(\etaVec)+A(\etaVec).\label{EqnTheoremPseudo}
\end{align}
Importantly, due to Eqn.~\ref{EqnTheoremPseudo}, the pseudo entropy of an exponential family distribution can be computed in closed form.
\begin{proposition}[General Exponential Family Mixtures]
    \label{prop:mixturesgen}
%
%
%
Consider an EF mixture model of Definition~\ref{def:mixtures} for which we do not require the observable distribution to have a constant base measure. As for Prop.\,\ref{prop:mixtures}, let the function $\etaVec(\cdot)$ map the standard parameters of the noise model to the distribution's natural parameters. 
If the Jacobian of $\etaVec(\cdot)$ is everywhere invertible, then the model is an EF model (see Section~\ref{SecIntro}) and the parameterization criterion (Definition~\ref{def:Param_Crit}) is fulfilled. It then applies at all stationary points:
\begin{eqnarray}
%
%
\hspace{8mm}\FFt(\Phi,\Theta) &=&  \frac{1}{N}\sum_{n=1}^N \HH[\qPhiN(c)]  \phantom{i} \phantom{\sum_{n=1}^N}\hspace{-3ex}-\ \HH[\,\mathrm{Cat}(c\,;\vec{\pi})\,] \phantom{i}   -\ \EEE{\;\qBar_{\Phi}(c)}{ \HHt[\,p(\xVec\,|\,\ThetaVec_c)] }\,. \phantom{\small{}ix} 
%
%
%
%
\label{EqnPropMixturesGEN}
%
%
%
%
%
%
\end{eqnarray}
where the dependence of the aggregate posterior $\qBar_\Phi(c)=\frac{1}{N}\sum_n \qn_\Phi(c)$ on $c$ is again denoted explicitly.
%
\end{proposition}
\begin{proof}
As for the proof in Prop.\,\ref{prop:mixtures}, we use $\PsiVec=(\pi_1,\ldots,\pi_{C-1})^{\mathrm{T}}$ for the prior parameters, and $\ThetaVec=(\ThetaVec_1,\ldots,\ThetaVec_C)^{\mathrm{T}}$
for the noise model parameters . Also as for the previous proof, $p_{\zetaVec(\PsiVec)}(c)$ denotes
the prior distribution, and $p_{\etaVec(c;\ThetaVec)}(\xVec)$ denotes the noise distribution.
As prior and noise distributions, $p_{\zetaVec}(c)$ and $p_{\etaVec}(\xVec)$ are exponential family distributions, the here assumed mixture model
is an EF model (see Section~\ref{SecIntro}).

To show that the mappings $\zetaVec(\PsiVec)$ and $\etaVec(c;\ThetaVec)$ fulfill the parameterization criterion (Definition~\ref{def:Param_Crit}), we
proceed precisely as for the proof in Prop.\,\ref{prop:mixtures}. As nowhere in the proof we require constant base measure for observables, we can
conclude that the parameterization criterion is fulfilled also under the weakened assumptions of Prop.\,\ref{prop:mixturesgen}.

For a general EF model with a potentially non-constant base measure that fulfills the parameterization criterion, Theorem~2 from \cite{LuckeWarnken2024} applies. As a result, the entropy decomposition presented in Eqn.~\ref{EqnTheoremPseudo} also applies, leading to the following expression:
\begin{eqnarray}
%
%
\hspace{5mm}\FFt(\Phi,\PsiVec,\ThetaVec) &=&  \frac{1}{N}\sum_{n=1}^N \HHt[\qPhiN(c)]  \phantom{ii} \phantom{\sum_{n=1}^N}\hspace{-3ex}-\ \HHt[\,\mathrm{Cat}(c\,;\vec{\pi})\,] \phantom{ii}   -\ \EEE{\;\qBar_{\Phi}}{ \HHt[\,\pT(\xVec\,|\,c)] }\,, \phantom{\small{}ix}
\label{EqnTheoremSoEGENDeri}
\end{eqnarray}
where the prior is a categorical distribution for a mixture model as defined in Definition~\ref{def:mixtures}.
\\
\end{proof}

\noindent{}We can now discuss some concrete examples of EF mixtures with constant and non-constant base measures. 
\subsection{Mixture of Gamma Distributions}
\label{SecGMM}
As an example of a mixture for which Prop.~\ref{prop:mixtures} applies, let us use a mixture of Gamma distributions defined for scalar observables $x\in\RRR$ for simplicity. A mixture component $c$ can then be written as exponential family distribution with $L=2$ parameters given by:
\begin{eqnarray}
p_{\etaVec(\ThetaVec_c)}(x) &=& h(x) \exp\big( \etaVecT(\ThetaVec_c)\, \TVec(x) \,-\, A(\ThetaVec_c)  \big)
%
%
%
\label{EqnExpGamma}
\end{eqnarray}
with $\ThetaVec_c=\left(\begin{array}{c} \alpha_c \\[0ex] \beta_c  \end{array}\right)$ and $\etaVec(\ThetaVec_c)= \left(\begin{array}{c} \alpha_c - 1 \\[0ex] - \beta_c  \end{array}\right)$. The base measure is constant, $h(x)=1$, and $A(\ThetaVec_c)=\log\big( \Gamma( \alpha_c ) \big) - \alpha_c \log \big(  \beta_c \big)$ is the log-partition function.
The Jacobian of $\etaVec(\cdot)$ is the $(2\times{}2)$-matrix $\left(\begin{array}{cc} 1 & 0 \\[0ex] 0 & -1 \end{array}\right)$ which is everywhere invertible. Consequently, Prop.\,\ref{prop:mixtures} applies for mixtures of Gamma distributions. Going back to standard notation (with $\Theta=(\piVec,\alphaVec,\betaVec)$ denoting all parameters of the gamma mixture), we obtain:
\begin{align}
%
%
&\LL(\Theta) \geq \FF(\Phi,\Theta)\\
 &=  \frac{1}{N}\sum_{n=1}^N \HH[\qPhiN(c)] - \HH[\,\mathrm{Cat}(c\,;\vec{\pi})\,] - \sum_{c=1}^C \qBar_{\Phi}(c)\, \HH[\mathrm{Gam}(x;\alpha_c,\beta_c)]\,, \phantom{\small{}ix}\label{EqnFFGammaMixture} \\
&\hspace{-0mm}\mbox{where\ \ \ } \HH[\mathrm{Gam}(x;\alpha_c,\beta_c)] \,=\,
\alpha_c-\log (\beta_c)+\log \big(\Gamma(\alpha_c)\big)+(1-\alpha_c)\psi(\alpha_c)
%
%
\end{align}
is the (well-known) entropy of the Gamma distribution. Notably, if the chosen $\qPhiN(c)$ are analytically computable, all terms of (\ref{EqnFFGammaMixture}) are easy and analytically computable (while the last term involves Gamma and di-gamma functions). 
%
%
%
%
%
%
\subsection{Mixture of Poisson Distributions}
\label{SecPMM}
As an example of a mixture for which Prop.\,\ref{prop:mixturesgen} applies but not Prop.~\ref{prop:mixtures}, 
let us consider a mixture of Poisson distributions for a $D$ dimensional observed space $\xVec\in\RRR^\mathrm{D}$. We get an exponential family distribution 
\begin{eqnarray}
p_{\etaVec(\ThetaVec_c)}(\xVec) &=& h(\xVec) \exp\big( \etaVecT(\lambdaVec^{(c)})\, \TVec(\xVec) \,-\, A(\lambdaVec^{(c)})  \big)\,,
%
%
%
\label{EqnExpPoisson}
\end{eqnarray}
where $\ThetaVec_c = \lambdaVec^{(c)} = \left(\lambda^{(c)}_1,\ldots, \lambda^{(c)}_D\right)^\mathrm{T}$ to parameterize the Poisson means of the $D$ scalar observables given latent~$c$. For all $c$ and $d$, we assume $\lambda^{(c)}_d > 0$.
The mapping $\etaVec(\lambdaVec^{(c)})$ for the Poisson distribution is defined by $\eta_d(\lambdaVec^{(c)})=\log(\lambda^{(c)}_d)$, sufficient statistics is $\TVec(\xVec)=\xVec$, and the log-partition function is given by $A(\lambdaVec^{(c)})=\sum_d \lambda^{(c)}_d$. Importantly in our context, the base measure $h(\xVec)$ is given by $h(\xVec)=\prod_{d=1}^D \big(x_d!)^{-1}$. 

The Jacobian of $\etaVec(\lambdaVec)$ is the diagonal $(D\times{}D)$-matrix given by
\begin{eqnarray}
 \delll{\etaVec(\lambdaVec)}{\lambdaVec^\mathrm{T}}  &=& \left(\begin{array}{ccc} \big(\lambda_1\big)^{-1} & \\[0ex]  & \ddots \\ & & \big(\lambda_D\big)^{-1} \end{array}\right)\,,
\end{eqnarray}
which is everywhere invertible for $\lambda_d>0$. Consequently, Prop.\,\ref{prop:mixturesgen} applies and we obtain:
\begin{eqnarray}
\hspace{10mm}\FFt(\Phi,\Theta) =  \frac{1}{N}\sum_{n=1}^N \HH[\qPhiN(c)]  \phantom{i} \phantom{\sum_{n=1}^N}\hspace{-3ex}-\ \HH[\,\mathrm{Cat}(c\,;\vec{\pi})\,] \phantom{i}   -\ \EEE{\;\qBar_{\Phi}(c)}{ \HHt[\,\mathrm{Pois}(\xVec\,|\,\lambdaVec^{(c)})] }\,. 
\end{eqnarray}
The pseudo entropy of the Poisson observable distribution can be easily computed with Eqn.~\ref{EqnTheoremPseudo} and is given by:
\begin{eqnarray}
\HHt[\,\mathrm{Pois}(\xVec\,|\,\lambdaVec)] &=& \sum_{d=1}^D \lambda_d\,\big(1-\log(\lambda_d)\big)\,.
\label{EqnPHPoisson}
\end{eqnarray}
Furthermore, according to Eqn.~\ref{EqnPseudoELBO} we get a lower bound of the standard log-likelihood $\LL(\PsiVec,\ThetaVec)$ if we add $\frac{1}{N}\sum_n \log\big( h(\xVecN)\big)$ to the alternative objective $\tilde{\FF}(\Phi,\Theta)$. Going back to standard parameterization with $\Theta=(\piVec,\lambdaVec^{(1)},\ldots,\lambdaVec^{(C)}\big)$, we thus arrive at:
\begin{align}
\LL(\Theta) &\geq \!\frac{1}{N}\sum_{n=1}^N \HH[\qPhiN(c)]  - \HH[\mathrm{Cat}(c;\vec{\pi})\,] \\
&\hspace{10mm}-\sum_{c=1}^C \qBar_{\Phi}(c)\HHt[\mathrm{Pois}(\xVec;\lambdaVec^{(c)})]
- \frac{1}{N}\sum_{n=1}^N\sum_{d=1}^D \log\big( x^{(n)}_d !\big).\nonumber
\end{align}
Notably, the lower bound of $\LL(\Theta)$ on the right-hand-side is computable in closed-form if the chosen variational distributions $\qPhiN(c)$ are
closed-form (which they essentially always are). Closed-form computability is enabled 
by the pseudo entropy which can be computable in closed-form (see Eqn.\,\ref{EqnPHPoisson}), while the standard entropy 
of the Poisson distribution involves an infinite~sum. 
%



\section{Discussion}
We have considered a range of prominent generative models and classes of generative models that can all be optimized using
an ELBO objective. Once a stationary point of the ELBO is reached during optimization, we have shown for these models
that the ELBO takes on the form of an entropy sum.


The obtained expressions for the ELBO at stationary points are more concise than the expression for the original ELBO objective, and usually much easier to compute.
In the case of GMMs, for instance, it is computationally more efficient to compute the entropy sum than the computation
of the ELBO itself. As a remark, for Poisson mixture models (PMMs), although Poisson entropies involve infinite sums, the pseudo-entropies in that case required for the entropy sums are closed-form. For models such as probabilistic PCA, 
entropy sums take on still more concise forms: the entropy sum is solely computable based on the model parameters alone (Eqn.\,\ref{EqnPropPCA}), i.e., no summation over data points is required. 
%

Future work can further extend the list of generative models for which the ELBO is equal to entropy sums at stationary points.
For the models treated here, the application of Theorem~1 or~2 of \cite{LuckeWarnken2024}, is relatively straight-forward.
However, some models (such as probabilistic PCA) 
have to be parameterized in a non-standard way for the
theorems to be applicable. 

Once entropy sum results are obtained, their usually concise form can be leveraged for different practical and theoretical investigations. For instance,
entropy sums can be used to analyze learning in generative models (see \cite{DammEtAl2023}), for model selection (see \cite{DammEtAl2023}), or they can be used to derive entropy-based learning objectives (see \cite{VelychkoEtAl2024}). 
We hope that the results for the concrete models here provided can be useful in these and many further contexts. Especially the very concise forms of ELBOs of as
prominent models as GMMs, SBNs, or probabilistic PCA may inspire future applications of entropy sums in theory and practice. Although GMMs,  
prob.\ PCA and their optimizations have intensively been studied, the here derived concise forms
of their ELBOs at stationary points have (to the knowledge of the authors) been unknown, so far.
%
%
\\

\noindent\small{{\bf Acknowledgement.} We acknowledge funding by the German Research Foundation (DFG) under project~464104047, `On the Convergence of Variational Deep Learning to Sums of Entropies', within the priority program `Theoretical Foundations of Deep Learning' (SPP 2298).}
%

\begin{thebibliography}{19}
\providecommand{\natexlab}[1]{#1}
\providecommand{\url}[1]{\texttt{#1}}
\expandafter\ifx\csname urlstyle\endcsname\relax
  \providecommand{\doi}[1]{doi: #1}\else
  \providecommand{\doi}{doi: \begingroup \urlstyle{rm}\Url}\fi
  
\bibitem[Aneja et~al.(2021)Aneja, Schwing, Kautz, and Vahdat]{AnejaEtAl2021}
J.~Aneja, A.~Schwing, J.~Kautz, and A.~Vahdat.
\newblock A contrastive learning approach for training variational autoencoder priors.
\newblock In \emph{Advances in Neural Information Processing Systems}, volume~34, pages 480--493, 2021.
  
\bibitem[Banerjee et~al.(2005)Banerjee, Merugu, Dhillon, Ghosh, and Lafferty]{BanerjeeEtAl2005}
A.~Banerjee, S.~Merugu, I.~S. Dhillon, J.~Ghosh, and J.~Lafferty.
\newblock Clustering with bregman divergences.
\newblock \emph{Journal of Machine Learning Research}, 6\penalty0 (58):\penalty0 1705--1749, 2005.
  
\bibitem[Bartholomew et~al.(2011)Bartholomew, Knott, and Moustaki]{BartholomewKnott1987}
D.~J. Bartholomew, M.~Knott, and I.~Moustaki.
\newblock \emph{Latent Variable Models and Factor Analysis: A Unified Approach}.
\newblock John Wiley \& Sons, 3rd edition, 2011.
  
\bibitem[Connor et~al.(2021)Connor, Canal, and Rozell]{ConnorEtAl2021}
M.~Connor, G.~Canal, and C.~Rozell.
\newblock Variational autoencoder with learned latent structure.
\newblock In \emph{International Conference on Artificial Intelligence and Statistics}, volume 130, pages 2359--2367. PMLR, 2021.
  
\bibitem[Damm et~al.(2023)Damm, Forster, Velychko, Dai, Fischer, and Lücke]{DammEtAl2023}
S.~Damm, D.~Forster, D.~Velychko, Z.~Dai, A.~Fischer, and J.~Lücke.
\newblock The {ELBO} of variational autoencoders converges to a sum of entropies.
\newblock In \emph{International Conference on Artificial Intelligence and Statistics}, volume 206, pages 3931--3960. PMLR, 2023.
  
\bibitem[Drefs et~al.(2023)Drefs, Guiraud, Panagiotou, and L{\"u}cke]{DrefsEtAl2023}
J.~Drefs, E.~Guiraud, F.~Panagiotou, and J.~L{\"u}cke.
\newblock Direct evolutionary optimization of variational autoencoders with binary latents.
\newblock In \emph{Proc. ECML 2022}, volume 13715 of \emph{LNCS/LNAI}, pages 357--372. Springer, 2023.
  
\bibitem[Everitt(1984)]{Everitt1984}
B.~Everitt.
\newblock \emph{An Introduction to Latent Variable Models}.
\newblock Chapman and Hall, 1984.
  
\bibitem[Hinton et~al.(2006)Hinton, Osindero, and Teh]{HintonEtAl2006}
G.~Hinton, S.~Osindero, and Y.~Teh.
\newblock {A fast learning algorithm for deep belief nets}.
\newblock \emph{Neural Computation}, 18:\penalty0 1527--1554, 2006.
  
\bibitem[Jordan et~al.(1999)Jordan, Ghahramani, Jaakkola, and Saul]{JordanEtAl1999}
M.~Jordan, Z.~Ghahramani, T.~Jaakkola, and L.~Saul.
\newblock An introduction to variational methods for graphical models.
\newblock \emph{Machine Learning}, 37:\penalty0 183--233, 1999.
  
\bibitem[Kingma and Welling(2014)]{KingmaWelling2014}
D.~P. Kingma and M.~Welling.
\newblock Auto-encoding variational bayes.
\newblock In \emph{International Conference on Learning Representations}, 2014.
  
\bibitem[L\"ucke and Warnken(2024)]{LuckeWarnken2024}
J.~L\"ucke and J.~Warnken.
\newblock On the convergence of the {ELBO} to entropy sums.
\newblock \emph{arXiv preprint arXiv:2209.03077}, 2024.
  
\bibitem[Makhzani et~al.(2016)Makhzani, Shlens, Jaitly, Goodfellow, and Frey]{MakhzaniEtAl2015}
A.~Makhzani, J.~Shlens, N.~Jaitly, I.~Goodfellow, and B.~Frey.
\newblock Adversarial autoencoders.
\newblock \emph{arXiv preprint arXiv:1511.05644}, 2016.
  
\bibitem[McLachlan and Peel(2004)]{McLachlanPeel2004}
G.~McLachlan and D.~Peel.
\newblock \emph{Finite mixture models}.
\newblock John Wiley \& Sons, 2004.
  
\bibitem[Neal and Hinton(1998)]{NealHinton1998}
R.~Neal and G.~Hinton.
\newblock A view of the {EM} algorithm that justifies incremental, sparse, and other variants.
\newblock In M.~I. Jordan, editor, \emph{Learning in Graphical Models}, pages 355--368. Kluwer, 1998.
  
\bibitem[Neal(1992)]{Neal1992}
R.~M. Neal.
\newblock Connectionist learning of belief networks.
\newblock \emph{Artificial intelligence}, 56\penalty0 (1):\penalty0 71--113, 1992.
  
\bibitem[Roweis(1998)]{Roweis1998}
S.~T. Roweis.
\newblock {EM} algorithms for {PCA} and {SPCA}.
\newblock In \emph{Advances in Neural Information Processing Systems}, volume~10, pages 626--632, 1998.
  
\bibitem[Tipping and Bishop(1999)]{TippingBishop1999}
M.~Tipping and C.~Bishop.
\newblock Probabilistic principal component analysis.
\newblock \emph{Journal of the Royal Statistical Society. Series B}, 61, 1999.
  
\bibitem[Tomczak and Welling(2018)]{TomczakWelling2018}
J.~Tomczak and M.~Welling.
\newblock {VAE} with a vampprior.
\newblock In \emph{International Conference on Artificial Intelligence and Statistics}, volume~84, pages 1214--1223. PMLR, 2018.
  
\bibitem[Velychko et~al.(2024)Velychko, Damm, Fischer, and L{\"u}cke]{VelychkoEtAl2024}
D.~Velychko, S.~Damm, A.~Fischer, and J.~L{\"u}cke.
\newblock Learning sparse codes with entropy-based {ELBO}s.
\newblock In \emph{International Conference on Artificial Intelligence and Statistics}, volume 238, pages 2089--2097. PMLR, 2024.
  
\end{thebibliography}

\end{document}